\documentclass[sigconf]{acmart}
\usepackage{booktabs} 
\setcopyright{rightsretained}

\usepackage{multirow}
\usepackage{booktabs} 
 \usepackage{csquotes}
 \usepackage{amsmath}
 \usepackage{amssymb}
 \usepackage{setspace}
 \usepackage{graphicx} 
 \usepackage{csquotes}
 \usepackage{mathtools}
 \usepackage{algorithm}
\usepackage[noend]{algorithmic}
\usepackage{caption}
\usepackage[algo2e,ruled,vlined,linesnumbered]{algorithm2e} 
\usepackage{microtype}
\usepackage{tikz}
	\usetikzlibrary{automata,arrows,positioning,calc}
\usepackage{algorithmic}
\usepackage{thmtools,thm-restate}
\usepackage{amssymb, amsthm}
\usepackage{textgreek}
\usepackage{subcaption}
\newcommand{\killpunct}[1]{}
\newcommand{\mathsym}[1]{{}}
\newcommand{\unicode}[1]{{}}
\newcommand*{\muGA}{($\mu$+1)~GA\xspace}
\newcommand*{\muga}{($\mu$+1)~GA\xspace}

\newcommand{\om}{\text{\sc OneMax}\xspace}

\newcommand{\onemax}{\text{\sc OneMax}\xspace}
\newcommand{\onemaxz}{\text{\sc OneMax}_z\xspace}
\newcommand{\jump}{\text{\sc Jump}\xspace}
\newcommand{\bo}[1]{\ensuremath{{\mathcal{O}(#1)}}}
\newcommand{\reals}{\ensuremath{\mathbb{R}}\xspace}
\newcommand{\N}{\ensuremath{\mathbb{N}}\xspace}
\newcommand{\pr}{\ensuremath{\mathbb{P}}\xspace}

\newcommand{\filtcond}[1]{\,;\,{#1}\mid\filt}
\newcommand{\filt}{\mathcal{F}_t}

\newtheorem*{informal}{\textbf{Informal}}

\acmDOI{10.1145/nnnnnnn.nnnnnnn}

\acmISBN{978-x-xxxx-xxxx-x/YY/MM}

\acmConference[GECCO '19]{the Genetic and Evolutionary Computation Conference 2019}{July 13--17, 2019}{Prague, Czech Republic}
\acmYear{2019}
\copyrightyear{2019}

\acmPrice{15.00}

\author{Dogan Corus}
\affiliation{
  \institution{Department of Computer Science\\ University of Sheffield}
  \city{Sheffield} 
  \postcode{S1 4DP, UK}
}
\email{d.corus@sheffield.ac.uk}

\author{Pietro S. Oliveto}
\affiliation{
  \institution{Department of Computer Science\\ University of Sheffield}
  \city{Sheffield} 
  \postcode{S1 4DP, UK}
}
\email{p.oliveto@sheffield.ac.uk}

\begin{document}
\title[Benefits of Populations on the Exploitation Speed of Steady-State GAs]{On the Benefits of Populations on the Exploitation Speed of Standard Steady-State Genetic Algorithms}


\begin{abstract}
It is generally accepted that populations are useful for the global exploration of multi-modal optimisation problems. Indeed, several theoretical results are available showing such advantages over single-trajectory search heuristics. In this paper we provide evidence that evolving populations via crossover and mutation may also benefit the optimisation time for hillclimbing unimodal functions. In particular, we prove bounds on the expected runtime of the standard ($\mu$+1)~GA for OneMax that are lower than its unary black box complexity and decrease in the leading constant with the population size up to $\mu=O(\sqrt{\log n})$. Our analysis suggests that the optimal mutation strategy is to flip two bits most of the time. To achieve the results we provide two interesting contributions to the theory of randomised search heuristics: 1) A novel application of drift analysis which compares absorption times of different Markov chains without defining an explicit potential function. 2) The inversion of fundamental matrices to calculate the absorption times of the Markov chains. The latter strategy was previously proposed in the literature but to the best of our knowledge this is the first time is has been used to show non-trivial bounds on expected runtimes.
\end{abstract}


\maketitle

  \section{Introduction}

Populations in evolutionary and genetic algorithms are considered crucial
for the effective global optimisation of multi-modal problems. For this to be the case,
the population should be sufficiently diverse such that it can explore multiple regions of the search space at the same time~\cite{FriedrichOlivetoSudholtWittECJ2009}.
Also, if the population has sufficient diversity,
then it considerably enhances the effectiveness of crossover for escaping from local optima.
Indeed the first proof that crossover can considerably improve the performance of GAs 
relied on either enforcing diversity by not allowing genotypic duplicates or by using unrealistically small crossover rates for the \jump function~\cite{Jansen2002}.
It has been shown several times that crossover is useful to GAs using the same, or similar, diversity enhancing mechanisms for a range of optimisation problems
including shortest path problems~\cite{Doerr2012}, vertex 
cover~\cite{Neumann2011}, colouring problems inspired by the Ising 
model~\cite{Sudholt2005} and computing input output sequences in finite state 
machines~\cite{Lehre2011}. 

These examples provide considerable evidence that,  by enforcing the necessary diversity, crossover makes GAs effective and often superior to applying mutation alone.
However, rarely it has been proven that the diversity mechanisms are actually necessary for GAs, or to what extent they are beneficial to outperform their mutation-only counterparts rather than being applied
to simplify the analysis. 
Recently, some light has been shed on the power of standard genetic algorithms without diversity over the same algorithms using mutation alone. 
Dang et al. showed that the plain \muga is at least a linear factor faster than its ($\mu$+1)~EA counterpart at escaping the local optimum of \jump~\cite{DangEtAlTEVC2018}.
Sutton showed that the same algorithm with crossover if run sufficiently many times is a fixed parameter tractable algorithm for the closest string problem while without crossover 
it is not~\cite{Sutton2018}.
Lengler provided an example of a class of unimodal functions to highlight the robustness of the crossover based version with respect to the mutation rate compared to the mutation-only version
i.e., the \muga is efficient for any mutation rate $c/n$ while the ($\mu$+1)~EA requires exponential time as soon as approx. $c>2.13$~\cite{Lengler2018}.
In all three examples the population size has to be large enough for the results to hold,
thus providing evidence of the importance of populations in combination with crossover.

Recombination has also been shown to be very helpful at exploitation if the necessary diversity is enforced through some mechanism.
In the (1+($\lambda,\lambda$))~GA such diversity is achieved through large mutation rates. The algorithm
can optimise the well-known \om function in sublogarithmic time with static offspring population sizes $\lambda$~\cite{doerr_tight_2015}, and in linear time with self-adaptive values of $\lambda$~\cite{DoerrAdaptive}.
Although using a recombination operator, the  algorithm is still basically a single-trajectory one (i.e., there is no 
population). More realistic steady-state GAs that actually create offspring by 
recombining parents 
have also been analysed for \om. 
Sudholt showed that ($\mu$+$\lambda$)~GAs are twice as fast as their 
mutation-only version (i.e., no recombination) for \om if diversity is 
enforced artificially i.e., genotype duplicates are preferred for deletion~\cite{Sudholthow}. He proved a runtime of  $(e/2) n \ln n 
+O(n)$ versus the  $e n \ln n +O(n)$ function evaluations required by any 
standard bit mutation-only evolutionary algorithm for \om and any other linear 
function~\cite{WittSTACS}. 
If offspring are identical to their parents it is not necessary to evaluate the quality of their solution. 
When the unnecessary queries are avoided,  the expected runtime of the GA using 
artificial diversity from \cite{Sudholthow} is bounded above by $(1+o(1)) 
0.850953 n \ln{n}$ \cite{DoerrCrossover2018}. 
Hence, it is faster than any unary (i.e., mutation-only)
unbiased ~\footnote{The probability of a bit being flipped by an unbiased 
operator is the same for each bit-position.} black-box search heuristic~\cite{LehreWitt2012}

On one hand, the enforced artificiality in the last two results considerably 
simplifies the analysis. On the other hand, the power of evolving populations 
for 
effective optimisation cannot be appreciated.  
Since the required diversity to 
make crossover effective is artificially enforced,  the optimal 
population size is 2 and larger populations provide no benefits. 
Corus and Oliveto showed that the standard \muga 
without diversity is still faster than mutation-only ones by proving 
an upper bound on the runtime of $(3/4) e n \ln n +O(n)$ for any $3< \mu < o(\log n / \log \log n)$~\cite{CorusOlivetoTEVC}. A result of 
enforcing the diversity in~\cite{Sudholthow} was that the best GA for the problem only used a 
population of size 2. However, even though this artificiality was removed in~
\cite{CorusOlivetoTEVC}, a population of size 3 was 
sufficient to get the best upper bound on the runtime achievable with their analysis. 
Overall, their analysis does not indicate any tangible benefit towards using a population larger than $\mu=3$. Thus, rigorously showing that populations are beneficial for GAs
in the exploitation phase has proved to be a non-trivial task.


In this paper we provide a more precise analysis of the behaviour of the population of the \muga for \om.
%
%
We prove that the standard \muga is at least  $60\%$ faster than the same algorithm using only mutation. 
We also prove that the GA is faster than any unary
unbiased 
 black-box search heuristic 
if offspring with identical genotypes to their parents are not evaluated. 
More importantly, our upper bounds on the expected runtime decrease 
with the population size up to $\mu =o(\sqrt{\log n})$, 
thus providing for the first time a natural example where 
{\it populations} evolved via recombination  and mutation optimise faster than unary unbiased heuristics. 


\section{Problem Definition and Our Results}

\subsection{The Genetic Algorithm}
\label{sec-preliminaries}
The \muGA is 
 a standard steady-state GA which samples a single 
new solution at every generation~\cite{EIBENSMITH,SarmaDeJongHANDBOOK}. It keeps a population of the $\mu$ best 
solutions sampled so far and at every iteration selects two solutions from the current 
population uniformly at random with replacement as the \emph{parents}. The 
recombination operator then picks building blocks from the parents to 
create the \emph{offspring} solution. For the case of pseudo-Boolean functions 
$f:\{0,1\}^n \rightarrow \reals$, the most frequently used recombination 
operator is \emph{uniform crossover} which picks the value of each bit 
position $i\in [n]$ from one parent or the other uniformly at random (i.e., from each parent with 
probability 1/2)~\cite{EIBENSMITH}. Then, an unbiased unary 
variation operator, which is called \emph{the mutation operator}, is applied to 
the offspring solution before it is added to the population. The most common 
mutation operator is the standard bit-mutation which independently flips each 
bit of the offspring solution with some probability $c/n$~
\cite{WittSTACS}. Finally, before moving to the next 
iteration, one of the solutions with the worst fitness value is removed from 
the population. For the case of maximisation the \muGA is defined in 
Algorithm~\ref{alg:mu+1-GA}. The runtime of Algorithm~\ref{alg:mu+1-GA} is the 
number of function evaluations until a solution which maximises the function 
$f$ is sampled for the first time. If every offspring is evaluated, then the 
runtime is equal to the value of the variable $t$ in Alg.~\ref{alg:mu+1-GA} when 
the optimal solution is sampled. However, if the fitness of offspring which are 
identical to their parents are not evaluated, then the runtime is smaller than 
$t$. We will first analyze the former  scheme 	and then adapt the result to the 
latter. 
 \begin{algorithm2e}[t] 
      \caption{\muGA \cite{EIBENSMITH,SarmaDeJongHANDBOOK}
    }
   \label{alg:mu+1-GA}
    $P_1 \gets \mu \textrm{ individuals, uniformly at random from } \{0, 
1\}^n$\; 
$t\gets \mu $\;
    \Repeat{\textrm{optimum is found}}
    {
            Select $x, y \in P_t$ uniformly at random with replacement \;
            $z \gets$ \emph{uniform crossover}$(x, y)$\;
            $z \gets mutate(z)$\;  
	    $P_{t+1} \gets P_t \cup \{z\}$\;
	    Remove the element with lowest fitness from $P_{t+1}$, breaking 
ties at random\; 
    $t \gets t+1$;
    }
    \end{algorithm2e}

\begin{figure}[!tbp]
    \includegraphics[width=.4\textwidth]{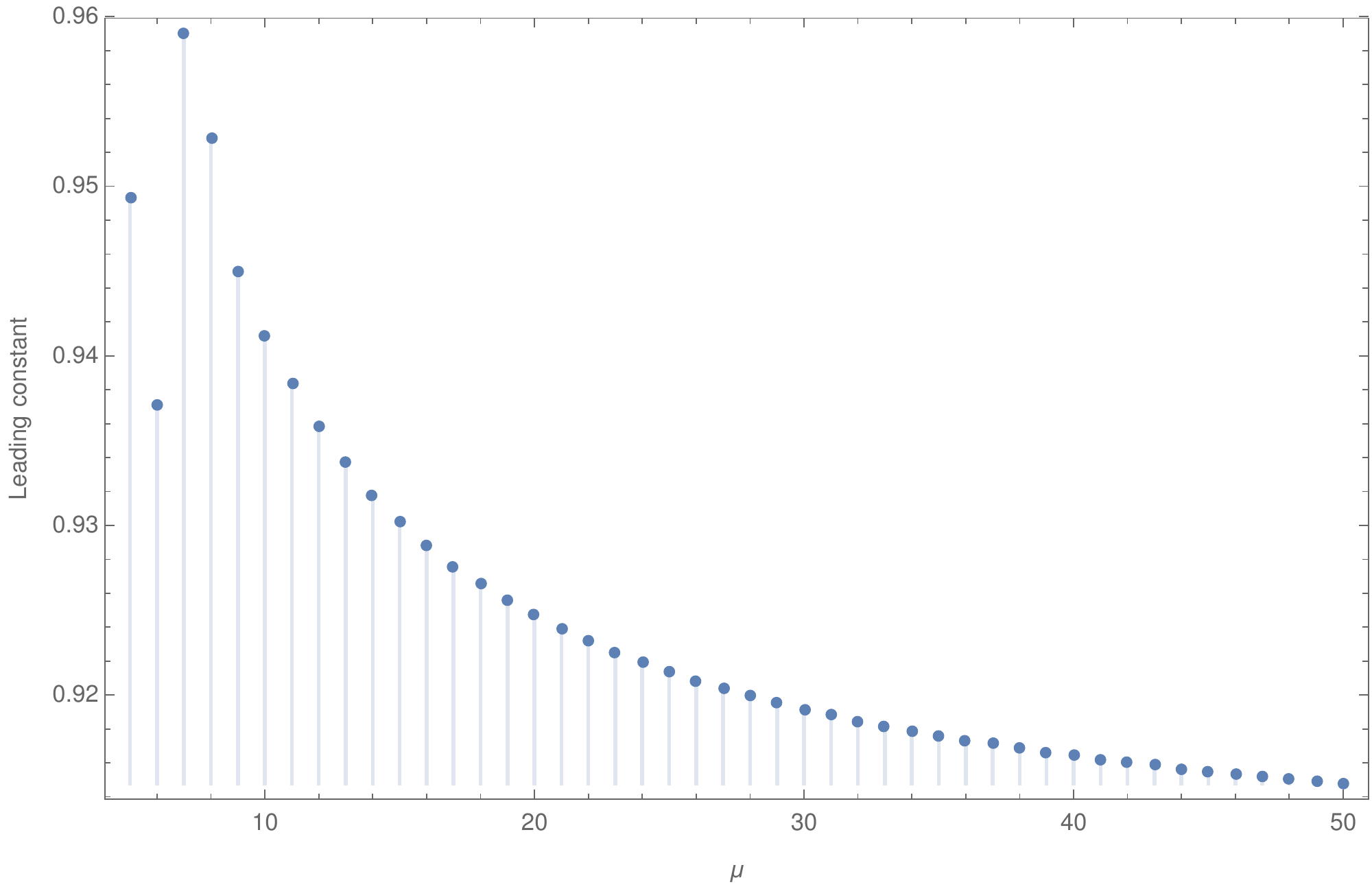}
    \caption{The leading constant from the second 
statement of Theorem~\ref{thm-main} versus the population size. The best 
leading constant achievable by any unary unbiased algorithm is $1$.}
    \label{fig:loc5to50}
  \end{figure}

\subsection{The Optimisation Problem}

Given a secret bitstring $z \in 
\{0,1\}^n$, $\onemax_z(x) :=\{ i \in [n] | z_i = x_i \}$ returns the number of 
bits on which a candidate solution $x \in \{0,1\}^n$ matches $z$ 
\cite{WittSTACS}. The 
\emph{optimisation time} (synonymously, \emph{runtime})  is defined by the 
number of queries to the function required by an algorithm to minimise the 
Hamming distance between the candidate solution and the hidden bitstring $z$. 

\subsection{Our Results}

In this paper we prove the following results.

\begin{informal}
 \label{thm-summary}
 The expected runtime for the \muga (
 with unbiased mutations and  population size 
$\mu=o(\sqrt{\log{n}})$ to optimise
the \onemax function is

\begin{enumerate}

\item  $E[T]\leq (1+o(1))n \ln{n}  \cdot \gamma_1 (\mu, p_0, p_1, p_2),$  
if offspring identical to their parents are not evaluated for 
their quality is known and $p_0$, $p_1$, and $p_2$ are respectively the 
probabilities that  zero, one or two bits are flipped, 
[Theorem~\ref{thm-main}, Section~\ref{sec-main}] 

\item $E[T]\leq (1+o(1))n \ln{n} \cdot \gamma_2(\mu, c),$  
 if the quality of all offspring is 
evaluated and using standard bit mutation 
with rate $c/n$, $c \in 
\Theta(1)$, [Corollary~\ref{cor-sbm}, Section~\ref{sec-main}]


\end{enumerate}
where $\gamma_1$ and $\gamma_2$  are decreasing functions 
of the population size $\mu$.
\end{informal}

  \begin{figure}[!tbp]
    \includegraphics[width=.4\textwidth]{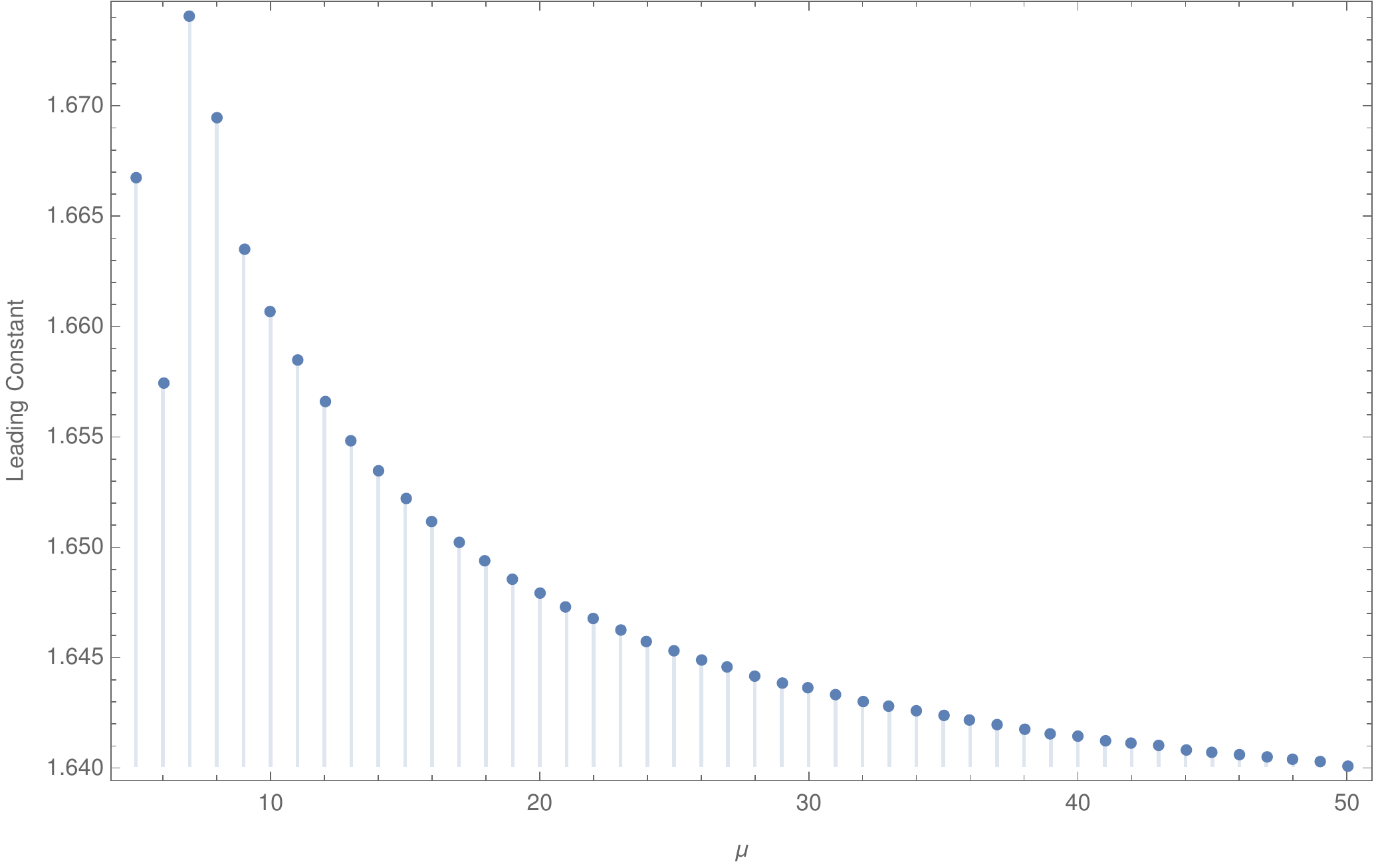}
    \caption{The leading constant from the first 
statement of Corollary~\ref{cor-sbm} versus the population size. The best 
mutation-only variant has a  the leading constant of $e\approx2.71$.}
    \label{fig:sbm5to50}
  \end{figure}

The above two statements are very general as they provide  upper bounds on the 
expected runtime of the \muga for each value of the population size up to 
$\mu=o(\sqrt{\log{n}})$ and any unbiased mutation operator. 
The leading constants $\gamma_1$ and $\gamma_2$ in Statements 1 and 2 are 
plotted respectively in Fig. 
\ref{fig:loc5to50} and \ref{fig:sbm5to50} for different population sizes using 
the  $p_0$, $p_1$, $p_2$ and $c$ values which minimise the 
upper bounds. 
The result is significant particularly for the following three reasons (in order of increasing importance).

\textbf{(1)} The first statement shows how the genetic algorithm 
outperforms any unbiased mutation-only heuristic since the best expected 
runtime achievable by any algorithm belonging to such class is at least
$n \ln n - cn \pm o(n)$ ~\cite{DoerrGecco2016}. Given that the best expected 
runtime achievable with any search heuristic using only standard bit mutation is 
$(1+o(1)) e n \ln{n}$~\cite{WittSTACS}, the second statement shows how by adding 
recombination a speed-up of 60\% is achieved for the \onemax 
problem for any population size up to $\mu=o(\sqrt{\log{n}})$. 

\textbf{(2)} Very few results are available proving constants in the leading  terms of the expected runtime for randomised 
algorithms due to the considerable technical difficulties in deriving them. 
Exceptions exist such as the analyses of
 \cite{WittSTACS} and 
\cite{DoerrGecco2016} without which our comparative results would not have been 
achievable.  
While such precise results are gaining increasing importance in the 
theoretical computer science community, the available ones are related to more 
simple algorithms. 
This is the first time similar results are achieved 
concerning a much more complicated to analyse standard genetic algorithm using 
realistic population sizes and recombination. 

\textbf{(3)} The preciseness of the analysis allows for the first time an 
appreciation of the surprising importance of the population for optimising 
unimodal functions~\footnote{Populations are traditionally thought to be useful 
for solving multi-modal problems.} as our upper bounds on the expected runtime 
decrease as the population size increases. In particular as the problem size 
increases, so does the optimal size of the population (the best known runtime 
available for the \muga was of  $(1+o(1)) 3/4 e n \ln n$ independent of the 
population size as long as it is greater than $\mu=3$ i.e., there were no 
evident advantages in using a larger population~\cite{CorusOlivetoTEVC}). This 
result is in contrast to all previous analyses of simplified evolutionary 
algorithms for unimodal functions where the algorithmic simplifications, made 
for the purpose of making the analysis more accessible, caused the use of 
populations to be either ineffective or  detrimental~\cite{Witt2006,Sudholthow,DoerrCrossover2018}. 
Our upper bound of  $\mu=o(\sqrt{\log{n}})$ 
is very 
close to the \emph{at most logarithmic} population sizes typically recommended 
for monotone functions to avoid detrimental 
runtimes~\cite{LehrePop2017,Witt2006}. We conjecture that the optimal population 
size is $\Theta(\log n^{1 - \epsilon})$ for any constant $\epsilon>0$,  which 
cannot be proven with our mathematical methods for technical reasons. 

\subsection{Proof Strategy}
Our aim is to provide a precise analysis of the expected runtime of the \muga 
for optimising \onemax with arbitrary problem size $n$. Deriving the exact 
transition probabilities of the algorithm from all possible configurations to 
all others of its population is prohibitive. We will instead 
devise a set of $n$ Markov chains, one for each improvement the algorithm has to 
pessimistically make to reach the global optimum, which will be easier to 
analyse. Then we will prove that the Markov chains are slower to reach their 
absorbing state than the \muGA is in finding the corresponding improvement. 

In essence, our proof strategy consists of: (1) to identify suitable Markov 
chains, (2) to prove that the absorbing times of the Markov chains are larger 
than the expected improving times of the actual algorithm and (3) to bound the 
absorbing times of each Markov chain.

In particular, concerning point (2) we will first define a potential function 
which monotonically increases with the number of copies of the 
genotype with most duplicates in the population and then bound 
the expected change in the potential function at every iteration (i.e., 
\emph{the drift }) from below. Using the maximum value of the potential function 
and the minimum drift, we will bound the expected time until the potential 
function value drops to its minimum value for the first time. This part of the 
analyses is a novel application of drift analysis 
techniques~\cite{lehrewittdrift}. In particular, 
rather than using an explicit distance function as traditionally occurs, we 
define the potential function to be equal to the conditional expected absorption 
time of the corresponding states of each Markov chain.

Concerning point (3) of our proof strategy, we will calculate the absorbing 
times of the Markov chains $M^j$ by identifying their fundamental matrices. This 
requires the inversion of tridiagonal matrices. Similar matrix manipulation 
strategies to bound the runtime of evolutionary algorithms have been previously 
suggested in the literature~\cite{HeYao,OlivetoBookchapter}. However, all 
previous applications showed that the approach could only be applied to prove 
results that could be trivially achieved via simpler standard methods. To the 
best of our knowledge, this is the first time that the power of this long 
abandoned approach has finally been shown by proving non-trivial bounds on the 
expected runtime.

 \section{Main Result Statement}
 \label{sec-main}
 Our main result is the following theorem. The transition probabilities  
$p_{i,k}$ for $(i,k)\in [m]^2$ and $m:=\lceil\mu/2\rceil$ are defined in 
Definition~\ref{def-mark} (Section~\ref{sec-mc}) and are depicted in 
Figure~\ref{fig-mchain}. 
 
 \begin{restatable}{theorem}{mainone}
 
 \label{thm-main}

 The expected runtime for the \muga with $\mu=o(\sqrt{\log{n}})$ using an 
unbiased mutation operator $mutate(x)$ that flips $i$ bits with probability 
$p_i$ with $p_0 \in \Omega(1)$ and $p_1 \in \Omega(1)$ to optimise the \onemax 
function is:

\begin{enumerate}
 
 \item $E[T]\leq (1+o(1))n \ln{n}   \frac{1}{p_{1} +  p_2 \frac{ 2(1-\xi_{2}) 
\mu}{(\mu+1)}     }$ if the quality of each offspring is evaluated, 
 
 \item $E[T]\leq (1+o(1))n \ln{n}   \frac{(1-p_{0}) }{p_{1} +  p_2 \frac{ 
2(1-\xi_{2}) \mu}{(\mu+1)}     }$ if the quality of offspring identical to their 
parents is not evaluated for their quality is known; and 
\end{enumerate}
\begin{align*}
&\xi_{i}=\frac{p_{i-1,i-2}}{p_{i-1,m}+p_{i-1,i-2} + p_{i-1,i}(1-\xi_{i+1})} ,
&\xi_{m}=\frac{p_{m-1,m-2}}{p_{m-1,m}+p_{m-1,m-2}}.
\end{align*}

\end{restatable}

The recombination operator of the GA is effective only if individuals with 
different genotypes are picked as parents (\emph{i.e.}, recombination cannot 
produce any improvements if two identical individuals are recombined). 
However, more often than not, the population of the \muga consists only 
of copies of a single individual. 
 When diversity is created via mutation  (\emph{i.e.}, a new genotype is added to the 
population), it either quickly leads to an improvement or it quickly 
disappears. The bound on the runtime reflects this behaviour as it is simply a 
waiting time until one of two event happens; either the current individual is 
mutated to a better one 
or diversity emerges 
and leads to an improvement before it is lost.

The  $\xi_2$ term in the runtime is the conditional probability that once 
diversity is created by mutation, it will be lost before reaching the next 
fitness level (an improvement). Naturally, $(1-\xi_2)$ is the 
probability that a successful 
crossover will occur before losing diversity. The $(1-\xi_2)$ 
factor increases with the population size $\mu$, which implies that larger 
populations have a higher capacity to maintain diversity long enough to be 
exploited by the recombination operator. 

Note that setting $p_i:=0$ for all $i>2$ minimises the upper bound on the 
expected runtime in the second statement of Theorem~\ref{thm-main} and reduces 
the bound to: $E[T]\leq (1+o(1))n \ln{n }(p_1 + p_2)/\left(p_{1}  +  p_{2}  
\frac{2\mu(1-\xi_{2})}{\mu+1} \right) $. Now, we can see the critical role that 
$\xi^*(\mu)=(1-\xi_2)\mu/(\mu+1)$ plays in the expected runtime. For any 
population size 
which yields $\xi^{*}(\mu)\leq 1/2$, flipping only one bit per mutation  becomes 
advantageous. The best upper bound achievable from the above expression is then 
$(1+o(1))n \ln{n}$ by assigning an arbitrarily small constant to  $p_0$ and  
$p_1=1-p_0$. As long as $p_0=\Omega(1)$, when an improvement occurs, the 
superior 
genotype takes over the population quickly relative to the time between 
improvements. Since there are only one-bit flips, the crossover operator 
becomes virtually useless (i.e., crossover requires a Hamming distance of 2 
between parents to create an improving offspring) and the resulting algorithm is 
a stochastic local search algorithm with a population. However, when $\xi^*(\mu)>1/2$ 
setting $p_2$ as large as possible provides the best upper bound. The  
transition for $\xi^*(\mu)$ happens between population sizes of 4 and 5. For 
populations larger than 5, by setting $p_1:=\epsilon/2$ and $p_0=:\epsilon/2$ to 
an arbitrarily small constant 
$\epsilon$ and setting $p_2=1-\epsilon$, we get the upper bound $  E[T]\leq 
(1+o(1))(1+\epsilon)n 
\ln{n}\left(\mu+1\right)/\left(2\mu\left(1-\xi_{2}\right)\right)$, which is 
plotted for different population sizes in Figure~\ref{fig:loc5to50}. 
A direct corollary to the main result is the upper bound for 
the classical \muga commonly used in evolutionary computation which 
applies standard bit mutation with mutation rate $c/n$ for which 
$p_0=(1-o(1))/e^c$, $p_1=(1-o(1))c/e^c$ and $p_2=(1-o(1))c^2/(2 e^c)$.

\begin{corollary}

\label{cor-sbm}

Let $\xi_2$ be as defined in Theorem~\ref{thm-main}. The expected runtime for 
the \muga with $\mu=o(\sqrt{\log{n}})$ using standard bit-mutation with 
mutation 
rate $c/n$, $c=\Theta(1)$ to optimise the \onemax function is:

\begin{enumerate}
 
 \item $E[T]\leq (1+o(1))n \ln{n}   \frac{e^c}{c+ \frac{ c^2 \mu}{(\mu+1)}    
(1-\xi_{2})}$ if the quality of each offspring is evaluated, 
 
 \item $E[T]\leq (1+o(1))n \ln{n}   \frac{(1-e^{-c})e^c}{c+ \frac{ c^2 
\mu}{(\mu+1)}    (1-\xi_{2})}$ if the quality of offspring identical to their 
parents is not evaluated for their quality is known;

\end{enumerate}

\end{corollary}

By calculating $\xi^{*}(\mu):=(1-\xi_2)\mu/(\mu+1)$ for fixed values of $\mu$ we can 
determine values of $c$ (i.e., mutation rate) which minimise the leading 
constant of the runtime bound in Corollary~\ref{cor-sbm}. In 
Figure~\ref{fig:sbm5to50} we plot the leading constants in the first 
statement, minimised by picking the appropriate $c$ values for $\mu$ ranging 
from 5 to 50. All the values presented improve upon the  upper bound on the 
runtime of $1.96 n \ln{n}$ given in \cite{CorusOlivetoTEVC} for any $\mu \geq 
3$ 
and $\mu = o(\log{n}/\log{\log{n}})$. All the upper bounds are less than $1.7 n 
\ln{n}$ and clearly decrease with the population size, signifying an at least 
$60\%$ increase in speed compared to the $e n 
\ln{n}\left(1-o\left(1\right)\right)$ lower bound for the same algorithm 
without the recombination operator.

Considering the leading constants in the second statement of 
Corollary~\ref{cor-sbm}, for all population sizes larger than 5, the upper bound 
for the optimal mutation rate is smaller than the theoretical 
lower bound on the runtime of unary unbiased 
black-box algorithms. For population sizes of 3 and 4,   $\xi^{*}=1/3$ and the 
expression to be minimised is $(1-e^{-c}) e^c/(c+ c^2/3)$. For $c>0$, this 
expression has no minimum and is always larger than one. Thus, at least with our 
technique, a population of size 5 or larger is necessary to prove that the \muga 
outperforms stochastic local search  and any other unary unbiased 
optimisation heuristic.

\section{Analysis}
 
Our main aim is to provide an upper bound on the expected runtime ($E[T]$) of 
the \muGA defined in Algorithm~\ref{alg:mu+1-GA} to maximise the \onemax 
function.
W.l.o.g. we will assume that the target string $z$ of the $\onemaxz$ function to 
be identified is the bitstring of all 1-bits since all the operators in the \muGA 
are invariant to the bit-value (have no bias towards $0$s or $1$s). We will 
provide upper bounds on the expected value $E[T^j]$, where $T^j$ is the time 
until an individual with at least $j+1$ 1-bits is sampled for the first time 
given that the initial population consists of individuals with $j$ 1-bits (\emph{i.e.}, the population is at level $j$).  
Then, by summing up the values of $E[T^j]$ and the expected times for the whole 
population to reach $j+1$ 1-bits for $j\in\{1\ldots,,n-1\}$ we achieve a valid 
upper bound on the runtime of the \muGA. Similarly to the analysis in 
\cite{CorusOlivetoTEVC}, we will pessimistically assume that the algorithm is 
initialised with all individuals having just 0-bits, and that throughout the 
optimisation process at most one extra 1-bit is discovered at a time. 

We will devise a Markov chain $M^j$ for each $j\in \{0, \ldots,n -1 \}$ 
for which we can analyse the expected absorbing time $E[T_{i}^{j}]$ starting 
from state $S_{i}^{j}$. We will then prove that it is in expectation slower in 
reaching its absorbing state than the \muGA is in finding an improvement given 
an initial population at level $j$. 
In particular, we will define a non-negative potential function 
on the domain of all possible configurations of a population at level $j$ or 
above. For any configuration at level $j$, we will refer to the genotype with 
the most copies in the population as the {\it majority genotype} and define the 
{\it diversity} of a population  as the number of non-majority individuals in 
the population. Our potential function will be monotonically decreasing with 
the diversity. Moreover, we will assign the potential function  a value of zero for 
all populations with at least one solution which has more than $j$ 1-bits.
%
Then, we will bound the expected change in the potential function at every 
iteration (i.e., \emph{the drift}) from below. Using the maximum value of the 
potential function and the minimum drift, we will derive a bound on 
the expected time until an improvement is found starting from a 
population at level $j$ with no diversity (\emph{i.e.}, all the solutions in the 
population are identical). While this 
upper bound will not provide an explicit runtime as a function of the problem 
size, it will allow us to conclude that the $E[T_{0}^{j}]\geq E[T^{j}]$. Thus, 
all that remains will be to bound the expected absorbing time of $M^j$ 
initialised at state $S_{0}^{j}$. We will obtain this bound by  identifying 
the fundamental matrix of $M^j$. After establishing that the inverse of the fundamental 
matrix is a \emph{strongly diagonally dominant tridiagonal matrix}, we will make 
use of existing tools in the literature for inverting such 
matrices and complete our proof.
%
%
%
%
%
\subsection{Markov Chain Definition}
\label{sec-mc}
 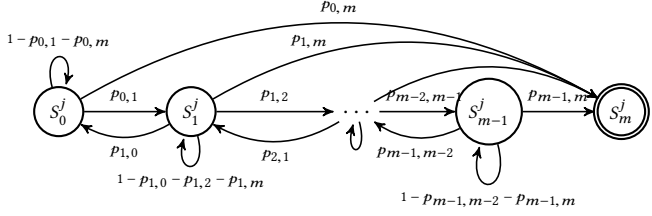
\begin{figure}\caption{The topology of Markov Chain $M^j$. } 
\label{fig-mchain}
  \begin{center}
    \begin{tikzpicture}[->, >=stealth', auto, semithick, node distance=2.2cm]
      \tikzstyle{every state}=[fill=white,draw=black,thick,text=black,scale=.8]

	  \node[state]    (B){$S^{j}_{0}$};
	  \node[state]    (C)[right of=B]   {$S^{j}_{1}$};
	  \node  (D)[right of=C]   {$\ldots$};
	  \node[state]    (E)[right of=D]   {$S^{j}_{m-1}$};
	  \node[state, double]    (F)[right of=E]   {$S^{j}_{m}$};
	  
	  \path
	    (B) edge[loop above]     	node{\tiny$1-p_{0,1}-p_{0,m}$}         
    (B);
	  \path
	    (B) edge[ left, above]	node{\tiny$p_{0,1}$}         (C);
	  \path
	    (C) edge[loop below]		
node{\tiny$1-p_{1,0}-p_{1,2}-p_{1,m}$}  		(C);
	  \path
	    (C) edge[bend left, below]     node{\tiny$p_{1,0}$}         (B);
	  \path
	    (C) edge[right, above]     node{\tiny$p_{1,2}$}         (D);
	    \path
	    (D) edge[right, above]     node[pos=0.6]{\tiny$p_{m-2,m-1}$}       
  (E);
	    \path
	    (E) edge[bend left, below]     node{\tiny$p_{m-1,m-2}$}         
(D);
	    \path
	    (E) edge[right, above]     node{\tiny$p_{m-1,m}$}         (F);
	     \path
	    (E) edge[loop below]		
node{\tiny$1-p_{m-1,m-2}-p_{m-1,m}$}  		(E);
	    \path
	    (D) edge[loop below]     node{}         (D);
	  \path
	    (D) edge[bend left, below]     node{\tiny$p_{2,1}$}         (C);
	  \path
	    (B) edge[bend left, above]     node{\tiny$p_{0,m}$}         (F);
	  \path
	    (D) edge[bend left, below]     node{}         
(F);
	  \path
	    (C) edge[bend left, above]     
node[pos=0.25]{\tiny$p_{1,m}$}         (F);
  
    \end{tikzpicture}
  \end{center}
\end{figure}
In this subsection we will present the Markov chains which we will use to 
model the behaviour of the \muga. Each Markov chain $M^j$ has 
$m:=\lceil\mu/2\rceil$ transient states ($S^{j}_{0}, S^{j}_{1},\ldots, S^{j}_{m 
-1}$) and one absorbing state 
($S^{j}_{m}$ ) with the topology depicted in Figure~\ref{fig-mchain}. 
The states $S_{i}^j$ represent the amount of diversity in the 
population. 
Hence, eg., $S_1^j$ refers to a 
population where all the individuals have $j$ 1-bits and all  but one of them 
share the same genotype, while $S_{m-1}^j$ refers to a population where at most 
$m-1 = \lceil \mu/2 \rceil +1$ individuals are identical. 
Compared to the analysis presented in \cite{CorusOlivetoTEVC} that used Markov 
chains of only three states (i.e., no diversity, diversity, increase in 
1-bits), $M^j$ allows to control the diversity in the population more 
precisely, thus to show that larger populations are beneficial to the 
optimisation process.

\begin{definition} \label{def-mark}
Let $M_{j}$ be a Markov chain with $m:=\lceil\mu/2\rceil$ transient states 
($S^{j}_{0}, 
S^{j}_{1},\ldots, S^{j}_{m -1}$) and one absorbing state ($S^{j}_{m}$ ) with 
transition probabilities $p_{i,k}$ from state $S^{j}_{i}$ to state $S^{j}_{k}$ 
as follows:
\begin{align*}
p_{0,1}&:= \frac{\mu}{(\mu+1)} \frac{2 j (n-j) p_2   }{n^2  }, \quad
p_{0,m}:=\frac{ (n-j)p_1}{n  },\\ 
\quad p_{i, m}& :=2 \frac{i}{\mu} 
\frac{\mu-i}{\mu} 
\frac{p_0}{4} \quad \text{ if } i>0, 
\end{align*}
\begin{align*}
p_{1,2}&:= p_{0} \left( \left(\frac{1}{\mu}\right)^2 
\frac{\mu-1}{\mu+1} 
+2 \frac{1}{\mu} \frac{\mu-1}{\mu}\frac{1}{4}  \frac{\mu-1}{\mu+1}            
\right), \\
p_{1,0}&:=p_{0} \left( \left(\frac{\mu-1}{\mu}\right)^2 +2 \frac{1}{\mu} 
\frac{\mu-1}{\mu}\frac{1}{4} \right) \frac{1}{\mu+1}, \\
p_{i,i+1}&:= p_{0} \left( \left(\frac{i}{\mu}\right)^2 \min{ 
\left(\frac{\mu-i}{\mu+1}, 1/4\right)} +2 \frac{i}{\mu} 
\frac{\mu-i}{\mu}\frac{1}{4}  \frac{\mu-i}{\mu+1}\right) \quad \\
\text{ if }& 
m-1>i>1,
\\
p_{i,i-1}&:=p_{0} \left( \left(\frac{\mu-i}{\mu}\right)^2 +2 
\frac{i}{\mu} \frac{\mu-i}{\mu}\frac{1}{4} + \left(\frac{i}{\mu}\right)^2 
\frac{1}{16}\right) \frac{i}{\mu+1}\\
  \quad \text{ if } &
m>i>1, \\
p_{m,m}&:=1, \quad p_{0,0}:= 1-p_{0,1}-p_{0,m},\\
%
%
p_{m-1, m-1}&:= 1-p_{m-1,m-2}-p_{m-1,m}, \\
p_{i, i}&:= 1-p_{i,i-1}-p_{i,i+1}-p_{i,m} \quad \text{ if } 0<i<m-1, \\
p_{i,j}&:=0 \quad otherwise.
\end{align*}
\end{definition}

Now, we will point out the important characteristics of these transition 
probabilities. The transition probabilities, $p_{i,k}$, are set to be equal to 
provable bounds on the probabilities of the \muga with a population consisting 
of solutions with $j$ bits of gaining/losing diversity ($p_{i,i+1}$/$p_{i,i-1}$) 
and sampling a solution with more than $j$ 1-bits ($p_{i,m}$). In particular, 
upper bounds are used for the transition probabilities $p_{i,k}$ where $i<k$ and 
lower bounds are used for the transition probabilities $p_{i,k}$ where $i>k$. 
Note that greater diversity corresponds to a higher probability of two distinct 
individuals with $j$ 1-bits being selected as parents and improved via 
recombination (i.e., $p_{i,m}$  monotonically increases with $i$  and  
recombination is ineffective if $i=0$ and the improvement probability $p_{0,m}$ 
is simply the probability of increasing the number of 1-bits by mutation only. 
Thus, $p_{0,m}=\Theta(j/n)$ while $p_{i,m}=\Theta(i(\mu-i)/\mu^2)$ when $i>0$.  
The first forward transition probability $p_{0,1}$ denotes the probability of 
the mutation operator of creating a different individual with $j$ 1-bits and the 
selection operator removing one of the majority individuals from the population. 
The other transition probabilities, $p_{i,i+1}$ and $p_{i,i-1}$ bound the 
probability that a copy of the minority solution or the majority solution is 
added to the population and that a member of the other species 
(minority/majority) is removed in the subsequent selection phase. All transition 
probabilities except $p_{0,1}$ and $p_{0,m}$ are independent of $j$ and referred 
to in the theorem statements without specifying $j$.

\subsection{Validity of the Markov Chain Model}

In this subsection we will present how we establish that $M^j$ is a pessimistic 
representation of Algorithm~\ref{alg:mu+1-GA} initialised with a population 
of $\mu$ identical individuals at level $j$. In particular, we will show that 
$E[T^{j}_0]$, the expected absorbing time starting from state $S_{0}^j$ is 
larger than $E[T^j]$. This result is formalised in the following lemma.

\begin{lemma}
\label{lem-relate}
Let $E[T]$ be the expected runtime until the \muga with 
$\mu=o(\log{n}/\log{\log{n}})$ optimises the \onemax function and   
$E[T_i^{j}]$ 
(or $E[T_i]$ wherever $j$ is obvious)  is the expected 
absorbing time of $M^j$ starting from state $S^{j}_{i}$. Then, $E[T] \leq o(n 
\log{n}) + (1+o(1))\sum_{j=0}^{n-1} E[T_{0}^{j}]$. 
\end{lemma}

We will use drift analysis~\cite{lehrewittdrift}, a frequently used tool in the 
runtime analysis of evolutionary algorithms, to prove the above result.  We 
will 
start by defining a 
potential function over the state space of Alg.~\ref{alg:mu+1-GA} that maps 
states to the expected absorbing times of $M^j$. The minimum of the potential 
function will correspond to the state of Algorithm~\ref{alg:mu+1-GA} which 
has sampled a solution with more than $j$ 1-bits and we will explicitly prove 
that the maximum of the potential function is $E[T^{j}_0]$. Then, we will show 
that the drift, \emph{i.e.},
the expected decrease in the potential function value in a single time 
unit (from time $t$ to $t+1$), is at least one. 
Using the maximum value of the potential function and the minimum drift, 
we will bound the runtime of the algorithm by the absorbing time of the Markov chain.

We will define our potential function over the domain of all possible 
population diversities at level $j$. We will refer to the genotype 
with the most copies in 
the population as the {\it majority genotype} and recall that the {\it 
diversity}, $D_t \in \{ 
0, \ldots, \mu-1\}$, of a population $P_t$ is defined as the number of 
non-majority individuals in the population. 

\begin{definition}\label{def-potential}
The potential 
function value for level $j$, $g^{j}$ (or $g$ wherever $j$ is obvious),  is defined as follows:
\[g^{j}(D_t)=g^{j}_{t}=
  \begin{cases} 
      E[T_{D_t}] &  0\leq D_t < \lceil\mu/2\rceil-1 \\
      E[T_{\lceil\mu/2\rceil-1 }] &  \lceil\mu/2\rceil-1 \leq D_t < \mu-1 \\
      0 & \exists x\in P_t \quad s.t. \quad \onemax(x)>j 
  \end{cases}
\]
where $E[T_i^j]$ (denoted as $E[T_i]$ wherever $j$ is obvious) is the 
expected absorbing time of the Markov chain $M^j$ starting from state 
$S^{j}_{i}$.
\end{definition}
The absorbing state of the Markov chain corresponds to a population with at 
least one individual with more than $j$ 1-bits, thus having  potential function 
value (and  expected absorbing time) equal to zero. The state $S_{0}^j$ corresponds to a population with no diversity.  The 
following lemma formalises that the expected absorbing time gets larger as the 
initial states get further away from $\lceil\mu/2\rceil-1$. This property 
implies that the expected absorbing time from state $S_{0}^j$  constitutes an 
upper bound for the potential function $g^j$. 
\begin{lemma} \label{lem-potbound}
 Let $E[T^{j}_i]$ be the expected absorbing time of Markov chain $M^j$ 
conditional on its initial state being $S_{i}^{j}$. 
Then,  $\forall j \in \{0,\ldots,n-1\}$, $E[T^{j}_i] \geq E[T^{j}_{i-1}]$ for all $ 1 \leq i \leq \lceil\mu/2\rceil$ and $g^{j}_{t}\leq E[T^{j}_{0}]$ $\quad \forall t>0 $. 
\end{lemma}

Now that the potential function is bounded from above, we will bound the drift
$E[g_{t}-g_{t+1})|D_t=i]$. Due to the law of total expectation, the expected 
absorbing time, $E[T_i]$ satisfies $\sum_j p_{i,j} (E[T_j]-E[T_i])=1$ for any 
absorbing Markov chain at state $i$.
Since $E[T_i]$ and $E[T_j]$ are the respective potentials of the states $S_i$ 
and $S_j$, the left hand side of the equation closely resembles the 
drift. 
Since the probabilities for $M^j$ are  
pessimistically set to underestimate the drift, the above equation allows us 
to formally prove the following: 
\begin{lemma} \label{lem-drift}
For a population at level $j$,
 $E[g_{t}^{j} - g_{t+1}^{j}| D_t=i] \geq 1-o(1) $ for all $t>0$
\end{lemma}

\begin{proof}
We will now show that $\Delta_t(i) := E[g_{t} - g_{t+1}| D_t=i]$, the 
expectation of the difference between the potential function values of 
population $P_t$ and $P_{t+1}$, is larger than one for all $i$. 
 
When there is no diversity in the population (\emph{i.e.}, $D_t=0$) the only 
way to increase the diversity is to introduce it during a mutation operation. 
A non-majority individual is obtained when one of the $n-j$ 0-bits and one of the $j$ 
1-bits are flipped while no other bits are touched. Then one of the majority 
individuals must be removed from the population during the selection phase.This 
event has probability at least
$
  p_2 \cdot 2 \cdot \frac{(n-j)}{n} \frac{ j}{n}  
\frac{\mu}{\mu+1}
= p_{0,1}.
$
Another way to change the potential function value is to create an improved 
individual with the mutation operator. In order to improve a solution it is 
sufficient to pick one of $n-j$ 1-bits and flip no other bits.  This event has 
the probability at least 
$
 p_{1} \cdot \frac{(n-j)}{n} =
p_{0,m}.
$
Thus, we can conclude that when $D_t=0$, the drift is at least $p_{0,m} (T_0) + 
p_{0,1} (T_0 - T_1)$. We can observe through the law of total expectation for 
the state $S^{j}_{0}$ of Markov chain $M^j$ that this expression for the drift 
when $D_t=0$ is larger than one. 

For $D_t>0$, we will condition the drift on whether the picked parents	
		are both majority individuals $\mathcal{E}_1$,
		are both minority individuals with the same genotype  
$\mathcal{E}_2$,
		are a pair that consists of one majority and one minority 
individual $\mathcal{E}_3$, or they
		are both minority individuals with different genotypes 
$\mathcal{E}_4$.

	Let $\mathcal{E}^*$ be the event that the population $P_t$ consists of 
two genotypes with Hamming distance two. Then, 
	\begin{align}\label{eq-parprob}
	&\pr\{\mathcal{E}_1|\mathcal{E}^*\} = 
\pr\{\mathcal{E}_1|\overline{\mathcal{E}^*}\}= 
\left(\frac{\mu-i}{\mu}\right)^2 \nonumber\\
	&\pr\{\mathcal{E}_2|\mathcal{E}^*\} 
=\pr\{\mathcal{E}_2|\overline{\mathcal{E}^*}\} + 
\pr\{\mathcal{E}_4|\overline{\mathcal{E}^*}\}= \left(\frac{i}{\mu}\right)^2 
\nonumber\\
	&\pr\{\mathcal{E}_3|\mathcal{E}^*\}=\pr\{\mathcal{E}_3 
|\overline{\mathcal{E}^*}\} = 2\frac{i}{\mu} \frac{\mu-i}{\mu}; \quad 
	\pr\{\mathcal{E}_4|\mathcal{E}^*\} = 0 
	\end{align}
	Note that when there are more than two genotypes in the population, the 
event of picking any two non-majority individuals is divided into two separate 
cases of picking identical minority individuals and picking two different 
minority individuals. Obviously, the sum of the probabilities of these two 
cases is equal to the probability of picking two minority individuals when there 
are only two genotypes (one majority and one minority) in the population. 

Restricting ourselves to $\Delta_{t}^{i>0}$, the drift conditional on  
$i>0$, the law of total expectation states: 
    \begin{align*}
     \Delta_{t}^{i>0}& = \pr\{\mathcal{E}^*\} \sum_{i=1}^{4} 
\left(\pr\{\mathcal{E}_i|\mathcal{E}^*\} 
\left(\Delta_{t}^{i>0}|\mathcal{E}_i,\mathcal{E}^*\right)\right) + \\
 + &(1-\pr\{\mathcal{E}^*\} )
\sum_{i=1}^{4} 
\left(\pr\{\mathcal{E}_i|\overline{\mathcal{E}^*}\} 
\left(\Delta_{t}^{i>0}|\mathcal{E}_i,\overline{\mathcal{E}^*}\right)\right) 
\end{align*}
We can rearrange the above expression using the probabilities from 
Eq~\ref{eq-parprob}
\begin{align*}
 \Delta_{t}^{i>0}& \geq  \left(\frac{\mu-i}{\mu}\right)^2 
\left(\Delta_{t}^{i>0}|\mathcal{E}_1\right) + \left(\frac{i}{\mu} 
\right)^2 
\min{\left(\left(\Delta_{t}^{i>0}|\mathcal{E}_2\right), 
\left(\Delta_{t}^{i>0}|\mathcal{E}_4\right) \right)}\\ +&2 
\frac{\mu-i}{\mu}\frac{i}{\mu} 
\left(\Delta_{t}^{i>0}|\mathcal{E}_3\right)
%
\end{align*}

We will now write the law of total expectation for state $i$ for our Markov 
chain $M^{j}$:
\[
1= p_{i, i+1} (E[T_{i}] -E[T_{i+1}]) + p_{i, i-1} (E[T_{i}] -E[T_{i-1}]) + 
p_{i, m} E[T_{i}]
\]

	We will then  substitute the probabilities in the law of total 
expectation with the values from Definition~\ref{def-mark}, 	
	\begin{align*}
	1&= p_{0} \left( \left(\frac{i}{\mu}\right)^2 \min{ 
\left(\frac{\mu-i}{\mu+1}, 1/4\right)} +2 \frac{i}{\mu} 
\frac{\mu-i}{\mu}\frac{1}{4}  \frac{\mu-i}{\mu+1}            \right) (E[T_{i}] 
-E[T_{i+1}]) \\
	&+ p_{0} \left( \left(\frac{\mu-i}{\mu}\right)^2 +2 
\frac{i}{\mu} \frac{\mu-i}{\mu}\frac{1}{4} + 
\left(\frac{i}{\mu}\right)^2 \frac{1}{16}\right) \frac{i}{\mu+1} (E[T_{i}] 
-E[T_{i-1}]) \\
	&+\left( 2 \frac{i}{\mu} \frac{\mu-i}{\mu} \frac{1}{4e^{c}} \right) 
E[T_{i}] 
	\end{align*}
	
	Finally, we will rearrange the above expression into the terms with 
the probabilities of events $\mathcal{E}_i$ as multiplicative factors
	\begin{align}
	1&= \left( \frac{\mu-i}{\mu}\right)^2 p_{0} \frac{i}{\mu+1} 
(E[T_i]-E[T_{i-1}])   \nonumber\\
	& +\left(\frac{i}{\mu}\right)^2  p_{0}\bigg( \min{ 
\left(\frac{\mu-i}{\mu+1}, 1/4\right)} (E[T_{i}] -E[T_{i+1}]) \nonumber\\ &+  
\frac{1}{16} 
\frac{i}{\mu+1}(E[T_i] -E[T_{i-1}])\bigg) \nonumber\\
	&+2 \frac{i}{\mu} \frac{\mu-i}{\mu} 
p_{0}\bigg(\frac{1}{4} \frac{\mu-i}{\mu+1} (E[T_{i}] -E[T_{i+1}]) \nonumber\\ 
&+ 
\frac{1}{4} \frac{i}{\mu+1} (E[T_{i}] 
-E[T_{i-1}])+\frac{E[T_i]}{4} 
 \bigg)\label{lawarranged1}
	\end{align}
	
We will refer to the first, second and third line of the
Eq~\ref{lawarranged1} as the $\mathcal{E}_1$, $\mathcal{E}_2$ and 
$\mathcal{E}_3$ term respectively. We will show that for each term, the 
conditional drift is larger than the term without the multiplicative factor.

When two majority individuals are selected as parents ($\mathcal{E}_1$), we 
pessimistically assume that improving to the next level and increasing the 
diversity has zero probability. Losing the diversity requires that no bits 
are flipped  during mutation and that a minority individual will be removed 
from 
the population. The probability that no bits are flipped is $p_0$. Thus we can show that:
$
\Delta_{t}^{i>0}|\mathcal{E}_1 \geq p_{0}\big(E[T_i] - E[T_{i-1}]\big) 
\frac{i}{\mu+1}
$.

This bound is obviously the same as the $\mathcal{E}_1$ term of 
Eq~\ref{lawarranged1} without the parent selection probability. 

When two minority individuals are selected as parents($\mathcal{E}_2$ or 
$\mathcal{E}_4$), if they are identical ($\mathcal{E}_2$) then it is sufficient 
that the mutation does not flip any bits which occurs with probability $p_0$ 
and that a majority 
individual is removed from the population (with probability $(\mu-i)/\mu $). 
Thus, the probability of increasing the diversity is $p_0 \times 
(\mu-i)/\mu $ and the probability of creating a majority individual is 
$\bo{1/n}$ since it is necessary to flip at least one particular bit position: 
$
\Delta_{t}^{i>0}|\mathcal{E}_2 \geq 
p_0\frac{\mu-i}{\mu}(E[T_i] - 
E[T_{i+1}])
$

However, if the two minority individuals have a Hamming distance of 
$2d\geq 2$ (\emph{i.e.}, $\mathcal{E}_4$), then in order to create another 
minority individual at the end of the crossover operation it 
is necessary that the crossover picks exactly $d$ 1-bits and $d$ 0-bits  among 
$2d$ bit positions where they differ. There are $\binom{2d}{d}$ different ways 
that this can happen and the probability that any particular outcome of 
crossover is realised is $2^{-2d}$. One of those outcomes though, might 
be the majority individual and if that is the case the diversity can decrease 
afterwards. However, while the Hamming distance between the minority 
individuals can be $2d=2$, obtaining a majority individual by 
recombining two minority individuals requires at least four specific bit 
positions to be picked correctly during crossover and thus does not occur with 
probability greater than $1/16$. On the other hand, when two different minority 
individuals are selected as parents,  there is at least  a $\frac{1- 
\binom{2d}{d} 2^{-2d}}{2}\geq 1/4$ 
probability that the crossover will result in an individuals with more 1-bits 
and then with probability $p_0 $ the mutation will not flip any 
bits. 
\begin{align*}
&\Delta_{t}^{i>0}|\mathcal{E}_{4} \geq p_0 \bigg[ 
\left(\binom{2d}{d}-1\right) 
2^{-2d} 
\frac{\mu-i}{\mu+1} (E[T_{i}] -E[T_{i+1}]) + \\ 
&+\left(\min\left(\frac{1}{16},2^{-2d}\right)+\bo{1/n}\right) \frac{i}{\mu+1} 
(E[T_{i}] -E[T_{i-1}])+ \frac{1}{4}E[T_i]  \bigg]\\
	&\geq \left(1-o(1)\right)\cdot p_0  \bigg[  
\min\left(\frac{1}{16},2^{-2d}\right) \frac{i}{\mu+1} 
(E[T_{i}] -E[T_{i-1}])+ \frac{E[T_i]}{4}  \bigg]
\end{align*}
Note that the $\mathcal{E}_{2}$ term of Eq~\ref{lawarranged1} multiplied with 
a factor of $(1-o(1))$ is smaller than both conditional drifts multiplied with 
the parent selection probability $(i/\mu)^2$. 
Finally, we will consider the drift conditional on event $\mathcal{E}_{3}$, 
the case when one minority and one majority individual are selected as 
parents. We will further divide this event into two subcases. In the first case 
 the Hamming distance $2d$ between the minority and the majority individual is 
exactly two ($d=1$). Then, the probabilities that crossover creates a copy 
of the minority individual, a copy of the majority individual or a new 
individual with more 1-bits are all equal to $1/4$. Thus, the conditional drift 
is $(\Delta_{t}^{i>0}|\mathcal{E}_{3},d=1)$
\begin{align*}
 \geq &  \frac{p_0}{4} \left( 
\frac{i}{\mu+1} 
(E[T_{i}]-E[T_{i-1}])+ \frac{\mu-i}{\mu+1}(E[T_{i}]-E[T_{i+1}]) + E[T_{i}] 
\right).
\end{align*}
However, when $d>1$, the drift is more similar to the case of $\mathcal{E}_{3}$ 
where the probabilities of creating copies of either the minority of the 
majority diminish with larger $d$ while larger $d$ increases the probability 
of creating an improved individual. More precisely, the drift is 

\begin{align*}
	&(\Delta_{t}^{i>0}|\mathcal{E}_{3},d>1)\geq p_0  
\bigg[ 
\left(\binom{2d}{d}-1\right) 2^{-2d} 
\frac{\mu-i}{\mu+1} (E[T_{i}] -E[T_{i+1}]) \\ &+ \left(2^{-2d}+\bo{1/n}\right) 
\frac{i}{\mu+1} (E[T_{i}] -E[T_{i-1}])
	+ (\frac{1- \binom{2d}{d} 2^{-2d}}{2})E[T_i]  \bigg]\\
	& \geq p_0  \bigg( 
\left(2^{-2d}+\bo{1/n}\right)\frac{i}{\mu+1} 
(E[T_{i}] -E[T_{i-1}])\\ &+ (\frac{1- \binom{2d}{d} 2^{-2d}}{2})E[T_i]  \bigg)\\
%
		 &\geq p_0  \left( 
\left(\frac{1}{16}+\bo{1/n}\right)\frac{i}{\mu+1} (E[T_{i}] -E[T_{i-1}])+ 
\frac{15}{32}E[T_i]  \right)\\
&\geq \left(1-o(1)\right)\cdot p_0  \left( 
\frac{1}{16}\frac{i}{\mu+1} (E[T_{i}] -E[T_{i-1}])+ 
\frac{15}{32}E[T_i]  \right)
\end{align*}
Since $(E[T_i] - E[T_{i-1}])$ is negative and $E[T_i] > E[T_i] - E[T_{i+1}]$, 
$(\Delta_t|\mathcal{E}_{3},d>1)\geq (\Delta_t|\mathcal{E}_{3},d=1)$. Now, 
finally we can observe that $(\Delta_t|\mathcal{E}_{3},d=1)$ multiplied with 
$2i(\mu-i)/\mu^2$ is larger than the $\mathcal{E}_{3}$ term in Eq~\ref{lawarranged1} 
multiplied with a factor of $(1-o(1))$. 

We have now shown piece by piece that the conditional drifts are larger 
than the corresponding terms in the right hand side of Eq~\ref{lawarranged1} 
up to small order terms, and thus established that $\Delta_t \geq (1-o(1))$ for all 
$t>0$. Since we have previously shown that $g^{j}_{t} \leq T_{0}^{j}$, we can now 
apply Theorem~\ref{thm-drift} to obtain $E[T^j] \leq (1+o(1)) E[T_{0}^{j}]$.

Once an individual with $j+1$ 1-bits is sampled for the first time it takes 
$\bo{\mu\log{\mu}}$ iterations before the whole population consists of 
individuals with at least  $j+1$ 1-bits \cite{Sudholthow, CorusOlivetoTEVC}. If 
the population 
size is in the order of $o(\log{n}/\log{\log{n}})$, then the total number of 
iterations where there  are individuals with different fitness values in the 
population is in the order of $o(n \log{n})$. Since $j \in \{0,1,\ldots,n-1\}$, 
we can establish that $E[T]\leq o(n\log{n})+\sum_{j=0}^{n-1}E[T^{j}] \leq 
o(n\log{n})+(1+o(1))\sum_{j=0}^{n-1} E[T_{0}^{j}]$.
\end{proof}

  With the potential function bounded from above by Lemma~\ref{lem-potbound} and 
the drift bounded from below by Lemma~\ref{lem-drift}, we can use 
the additive drift theorem\footnote{The additive drift theorem is 
provided in the appendix as Theorem~\ref{thm-drift} for reviewer convenience} 
from \cite{lehrewittdrift} to bound $E[T^j]$ by $E[T_{0}^j]$. By summing over 
all levels, we get the bound stated in Lemma~\ref{lem-relate} on the expected 
runtime of Algorithm~\ref{alg:mu+1-GA}.

\subsection{Markov Chain Absorption Time Analysis}
In the previous subsection we stated in Lemma~\ref{lem-relate} that we can 
bound the absorbing times of the Markov chains $M^j$ 
to derive an upper bound on the runtime of Algorithm~\ref{alg:mu+1-GA}.
In this subsection we use mathematical tools developed for the analysis of 
Markov chains to provide such bounds on the absorbing times.  

The absorbing time of a Markov chain starting from any initial state 
$i$ can be derived by identifying its fundamental matrix. 
Let the matrix $Q$ denote the transition probabilities between the transient 
states of the Markov chain $M^j$. The fundamental matrix of $M^j$ is 
defined as $N:=(I - Q)^{-1}$ where $I$ is the identity matrix. The most 
important characteristic of the fundamental matrix is that when it is 
multiplied by a column vector of ones, the product is a vector holding 
$E[T^{j}_i]$, the expected absorbing times conditional on the initial state 
$i$ of the Markov chain. Since, Lemma~\ref{lem-relate} only involves 
$T_{0}^{j}$, we are only interested in the entries  of the first row of 
$N=[n_{ik}]$. 
However, inverting the matrix $I-Q$ is not always a straightforward task. 
Fortunately,  $I-Q=[a_{ik}]$ has characteristics that allow bounds on the 
entries of its inverse. Its entries are related to the transition probabilities 
of $M^j$ as follows:
\begin{align}
 a_{11}&=1-p_{0,0}= p_{0,1}+ p_{0,m} \label{eq-matfirst}\\
 a_{mm}&=1-p_{m-1,m-1}= p_{m-1,m-2}+ p_{m-1,m}\\
 a_{ii}&=1-p_{i-1,i-1}= p_{i-1,i-2}+ p_{i-1,i}+p_{i-1,m} \nonumber\\ &\forall i 
\in\{2,...,m-1\}\\
 a_{ik}&=-p_{i-1,k-1} \quad \forall i,k \in \{1,\ldots,m\} \wedge i\neq k 
\label{eq-matlast}
\end{align}

Observe that $I-Q$  is a \emph{tridiagonal} matrix, in the sense that all 
non-zero elements of $I-Q$ are either on the diagonal or adjacent to it. 
Moreover, the diagonal entries $a_{ii}$ of $I-Q$ are in the form 
$1-p_{i-1,i-1}$, which is equal to the sum of all transition probabilities out 
of state $i-1$. Since the other entries on row $i$ are transition probabilities 
from state $i-1$ to adjacent states, we can see that $|a_{ii}| > \sum_{i\neq k} 
a_{ik}$. The matrices where $|a_{ii}| > \sum_{i\neq k} a_{ik}$ holds are 
called  \emph{strongly diagonally dominant} (SDD). Since $I-Q$ is SDD, according 
to Lemma~2.1\footnote{\label{note}For reviewer convenience, it is 
Lemma~\ref{imported} in 
the Appendix.} in \cite{li2010inverses}, it holds for the fundamental matrix 
$N$ 
for  all $i\neq k$ that, $|n_{i,k}| \leq |n_{k,k}| \leq \left(|a_{k,k}| 
\left(1-\frac{\sum_{l\neq j}|a_{l,k}| }{|a_{k,k}|}\right)\right)^{-1} \leq 
\left(|a_{k,k}|- \sum_{l\neq j}|a_{l,k}| \right)^{-1}$.

In our particular case, the above inequality implies that $|n_{1,k}| \leq  
1/p_{k-1,m}$. For any population with diversity, there is a probability in the 
order of $\bo{1/\mu}$ to select one minority and one majority individual and a 
constant probability that their offspring will have more 1-bits than the 
current level. Considering $m=\bo{\mu}$, $E[T_{0}^{j}] = \sum_{k=1}^{m}n_{1,k} < n_{1,1} + \sum_{k=2}^{m} 
\frac{1}{p_{k-1,m}} \leq  n_{1,1} + \bo{\mu^2}.$
%
We note here that the $\bo{\mu^2}$ factor in the above expression creates the 
condition $\mu=o(\sqrt{\log{n}})$ on the population size  for our main 
results. We will now bound the term $n_{1,1}$ from above to establish our upper 
bound using the following theorem:
\begin{theorem} [Direct corollary to Corollary 3.2 in \cite{li2010inverses}] 
\label{thm-maininv}
 $A$ is an $m\times m$ tridiagonal non-singular SDD matrix such that 
$a_{i,k}\leq 0$ for all $i\neq k$, $A^{-1}=[n_{i,k}]$ exists and $n_{i,k}\geq 
0$ for all $i,k$, then 
$n_{1,1}= 1/(a_{1,1}+ a_{1,2}\xi_{2})$, $
\xi_{i}=a_{i,i-1}/(a_{i,i}+ a_{i,i+1}\xi_{i+1}) $, and 
$\xi_{m}=a_{m,m-1}/a_{m,m}$.
\end{theorem}
In order to use Theorem~\ref{thm-maininv},  we need to satisfy its conditions. 
We can easily see that non-diagonal entries of the original matrix $I-Q$ are 
non-positive and use Theorem~3.1\footnote{\label{note}For reviewer 
convenience, it is Theorem~\ref{thm-signinv} in the Appendix.}
in \cite{li2010inverses} to show that 
$N=(I-Q)^{-1}$ has no negative entries. Thus, Theorem~\ref{thm-maininv} yields:



\begin{lemma}\label{lem-absbound}
 With an initial population of size $\mu=o(\sqrt{\log{n}})$ at level $j$, the 
expected time $E[T^j]$ until an individual with $j+1$ 1-bits is sampled by the \muga  
for the first time is bounded from above  as follows:
\begin{align*}
 &E[T^{j}_{0}] \leq  \frac{n}{n-j} \frac{1}{p_1+ 
\frac{2 \mu p_2}{(\mu+1)} 
\frac{j  }{n  } 
(1-\xi_{2})} +o(\log{n})\\ 
&\xi_{m}=\frac{p_{m-1,m-2}}{p_{m-1,m}+p_{m-1,m-2}}\\
&\xi_{i}=\frac{p_{i-1,i-2}}{p_{i-1,m}+p_{i-1,i-2} + p_{i-1,i}(1-\xi_{i+1})} 
 \quad  \forall 1<i<m=\lceil\mu/2\rceil .
\end{align*}
where $p_{i,k}$ are the transition probabilities of the Markov chain $M^j$.
\end{lemma}

The above bound on $E[T^{j}_{0}]$, together with 
Lemma~\ref{lem-relate}, yields Theorem~\ref{thm-main}, our main result.

\section{Conclusion}
In this work, we have shown that the steady-state \muga optimises 
\onemax faster than any unary unbiased search heuristic. 
Providing precise asymptotic bounds on the expected runtime of standard GAs without artificial mechanisms that simplify the analysis has been a long 
standing open problem. 
We have derived bounds up to the leading term constants of the expected runtime.
To achieve this result we show 
that a simplified Markov chain  pessimistically represents the behaviour of  
the GA for \onemax. This insight about the algorithm/problem pair allows the derivation 
of runtime bounds for a complex multi-dimensional stochastic 
process. The analysis shows that as the number of states in the Markov chain (the population size) increases, so does the 
probability that diversity in the population is kept. Thus, larger populations increase the probability that recombination 
finds improved solutions quickly, hence reduce the expected runtime. 

\bibliography{document}
\bibliographystyle{ACM-Reference-Format}
\newpage
\section{Appendix}

%

\subsection{Proof of Lemma~\ref{lem-potbound}}

We will first show that $E[T^{j}]\leq (1+o(1))E[T_{0}^{j}]$ .
In order to achieve this result, we will first establish that $g^{j}_{t}\leq 
E[T_{0}]\quad \forall t>0 $ and then that  $\Delta_t \geq 1-o(1)
1 \quad \forall t>0$.  These two results will allow us to use 
Theorem~\ref{thm-drift} and conclude that $E[T^{j}]\leq (1+o(1)) E[T_{0}^{j}]$. 
For the rest of the analysis of $E[T^{j}]$, we will drop the superscript of 
$E[T_{i}^{j}]$ and $g^{j}_{t}$.

\begin{proof}
 We are interested in $\max{(E[T_{0}], E[T_{1}], \ldots, E[T_{m}])}$ since 
these are the values that the potential function $g$ can have. According to the 
transition probabilities in Definition~\ref{def-mark}, $p_{i+1,m} \geq p_{i,m} 
$  for all $i$. Using this observation and the law of total expectation  we can 
show that not only $\max{(E[T_{0}], E[T_{1}], \ldots, E[T_{m}])}= E[T_{0}]$ but 
also $E[T_{i-1}] \geq E[T_{i}]$ for all $i$.  
First, we will prove that $E[T_{m-2}] \geq E[T_{m-1}]$ by contradiction. Then, 
we will prove by induction that $E[T_{i-1}] \geq E[T_{i}]$ for all $i$. For this 
induction we will use $E[T_{m-2}] \geq E[T_{m-1}]$ as our basic step and we will 
prove by contradiction that if for all $j>i$ $E[T_{j-1}]\geq E[T_{j}]$ holds 
then $E[T_{i-1}] \geq E[T_{i}]$ must also hold. 

If we use the law of total expectation for the absorbing time starting from 
state  $0<i<m$, we obtain:
\begin{align*}
&E[T_{i}] = p_{i, i+1} (E[T_{i+1}] +1 ) + p_{i, i-1} (E[T_{i-1}] +1 ) 
\\&+ p_{i, m+1} + (1- p_{i, i+1} - p_{i, i-1}- p_{i, m}) (E[T_{i}] +1 ) 
\end{align*}

This equation can be rearranged as follows:
\[
1= p_{i, i+1} (E[T_i] -E[T_{i+1}] ) + p_{i, i-1} (E[T_i] -E[T_{i-1}] ) 
+ p_{i,m} E[T_i] 
\]
For the special case of $i=m$, we have:
\[
1= p_{m-1, m} (E[T_{m-1}] ) + p_{m-1, m-2} (E[T_{m-1}] -E[T_{m-2}] ) 
\]
If we introduce the allegedly contradictory assumption $E[T_{m-2}] < 
E[T_{m-1}]$, the above equation implies:
\begin{align*}
&1 > p_{m-1, m} (E[T_{m-1}] )  \implies  \frac{1}{p_{m-2,m}} 
\geq \frac{1}{p_{m-1, m} }\\ &> E[T_{m-1}] > E[T_{m-2}] \\ &\implies 
\frac{1}{p_{m-2,m} }  > E[T_{m-2}] 
\end{align*}

Given that $\frac{1}{p_{i, m} }  > E[T_{i}]$  and $E[T_{i+1}] > E[T_{i}]$   the 
law of total expectation for $i$ implies:
\[
1= p_{i, i+1} (E[T_{i}] -E[T_{i+1}]) + p_{i, i-1} (E[T_{i}] -E[T_{i-1}]) + 
p_{i, m} E[T_{i}]
\]
\[
1< p_{i, i+1} (E[T_{i}] -E[T_{i+1}]) + p_{i, i-1} (E[T_{i}] -E[T_{i-1}]) + 1
\]
\[
0< p_{i, i-1} (E[T_{i}] -E[T_{i-1}]) 
\]
\[
E[T_{i-1}]< E[T_{i}]  \implies \frac{1}{p_{i-1, m} }\geq\frac{1}{p_{i, m} } > 
E[T_i]  
> E[T_{i-1}]
\]
Thus the allegedly contradictory claim $E[T_{m-1}] > E[T_{m-2}]$ induces over 
$i$ 
such that it implies $E[T_1] > E[T_0]$ and $1/p_{0,m} > E[T_0]$. We can now 
write the 
total law of expectation for $i=0$.
\[
1= p_{0,1} (E[T_{0}] - E[T_{1}]) + p_{0,m} E[T_{0}]
 \]
\[
1 < p_{0,1} (E[T_{0}] - E[T_{1}]) + 1
 \]
\[
0 < p_{0,1} (E[T_{0}] - E[T_{1}]) 
 \]
\[
0 > p_{0,1} 
 \]
The last statement is a contradiction since a probability cannot be negative. 
This contradiction proves the initial claim $E[T_{m-2}] \geq E[T_{m-1}]$ 

We will now follow a similar route to prove that $E[T_{i-1}] \geq E[T_{i}]$ for 
all 
$i$.
Given that for all $j>i$  $E[T_{j-1}] \geq E[T_{j}]$ holds, we will show that 
$E[T_i>E[T_{i-1}]$
creates a contradiction. We start with the law of total expectation for $E[T_i$:
\[
1= p_{i, i+1} (E[T_{i}] -E[T_{i+1}]) + p_{i, i-1} (E[T_{i}] -E[T_{i-1}]) + 
p_{i, m} E[T_{i}]
\]
Our assumption `` $\forall j>i \quad E[T_j] \leq E[T_{j-1}]$'' 
implies that $E[T_{i}]-E[T_{i+1}] \geq 0$, thus we obtain: 
\[
1> p_{i, i-1} (E[T_{i}] -E[T_{i-1}]) + p_{i, m} E[T_{i}]
\]
With our allegedly contradictory assumption $E[T_{i}]-E[T_{i-1}]>0$ we obtain:
\[
1> p_{i, m} E[T_{i}] \implies \frac{1}{p_{i-1, m} }  \geq \frac{1}{p_{i, m} } > 
E[T_i]  > E[T_{i-1}]
\]
We have already shown above that $1/p_{i-1, m}\geq 1/p_{i, m}  > E[T_i]  > 
E[T_{i-1}]$ can be 
induced over $i$ and implies $E[T_1] > E[T_0]$ and $1/p_{0,m} > E[T_0]$. Then 
we can 
conclude that: 
\begin{equation}\label{eq-abord}
E[T_i] \geq E[T_{i+1}]\quad \forall 0 \leq i \leq \lceil\mu/2\rceil-1.  
\end{equation}

The above conclusion implies that $E[T_0]$ is the largest value that our 
potential function can have and $E[T_{i}] - E[T_{i+1}]$ is non-negative 
for all $i$.
\end{proof}


\subsection{Proof of Lemma~\ref{lem-relate}}
Here, we will use the following additive drift 
theorem from \cite{lehrewittdrift},

\begin{theorem}[Theorem 3 in \cite{lehrewittdrift}]
 \label{thm-drift}
Let $(X_t)_{t\in\N_0}$, be a stochastic process, adapted to a filtration 
$(\filt)_{t\in\N_0}$, over some state space $S\subseteq \{0\}\cup 
[x_{min},x_{max}]$, where 
$x_{min}\ge 0$. 
Let $g:\{0\}\cup [x_{min},x_{max}]\to \reals^{\ge 0}$ be any function
such that $g(0)=0,$ and $0<g(x)<\infty$ for all 
$x\in[x_{min},x_{max}]\setminus\{0\}$.
Let $\mathcal{T}_a=\min\{t\mid X_t\le a\}$ for $a\in \{0\}\cup 
[x_{min},x_{max}]$. 
If $E[g(X_t)-g(X_{t+1})  \filtcond{ X_t\ge x_{min}}]\ge 
\alpha_{\mathrm{u}}$ 
for some $\alpha_{\mathrm{u}}>0$ then 
$E[\mathcal{T}_0\mid X_0] \le \frac{g(X_0)}{\alpha_{\mathrm{u}}}$.
\end{theorem}

\begin{proof}[Proof of Lemma~\ref{lem-relate}]
Since $\Delta_t \geq (1-o(1))$  from Lemma~\ref{lem-drift} and $g^{j}_{t} \leq 
E[T_{0}^{j}]$ from Lemma~\ref{lem-potbound}, we can apply 
Theorem~\ref{thm-drift} to obtain $E[T^j] \leq (1+o(1)) E[T_{0}^{j}]$.
Once an individual with $j+1$ 1-bits is sampled for the first time it takes 
$\bo{\mu\log{\mu}}$ iterations before the whole population consists of 
individuals with at least  $j+1$ 1-bits \cite{CorusOlivetoTEVC}. If 
the population size is in the order of $o(\log{n}/\log{\log{n}})$, then the 
total number of 
iterations where there  are individuals with different fitness values in the 
population is in the order of $o(n \log{n})$. Since $j \in \{0,1,\ldots,n-1\}$, 
we can establish that 
\[E[T]\leq o(n\log{n})+\sum_{j=0}^{n-1}E[T^{j}] \leq 
o(n\log{n})+(1+o(1))\sum_{j=0}^{n-1} E[T_{0}^{j}].\]
\end{proof}

\subsection{Proof of Lemma~\ref{lem-absbound}}

We made use of the following lemma to obtain .
\begin{equation}\label{eq-invbound}
E[T_{0}^{j}] = \sum_{k=1}^{m}n_{1,k} < n_{1,1} + \sum_{k=2}^{m} 
\frac{1}{p_{k-1,m}} \leq  n_{1,1} + \bo{\mu^2}.
\end{equation}
 \begin{lemma}[Lemma 2.1 in \cite{li2010inverses}] \label{lem-simbound}
\label{imported}
 Let $A$ be a SDD matrix, then $A^{-1}=[n_{ik}]$ exists and for any $i \neq k$,
 \[
  |n_{ik}|\leq  n_{kk}  \frac{\sum\limits_{i\neq k} |a_{ik}| }{|a_{ii}|}\leq 
 n_{kk}\leq \frac{1}{|a_{kk}| - \sum\limits_{ l \neq k} |a_{k l}| 
 }
 \]
 \end{lemma}

In order to show that the entries of the fundamental matrix of $M^j$ 
are all non-negative, we will employ the following result.

\begin{theorem}[Theorem 3.1 in \cite{li2010inverses}] \label{thm-signinv}
If $A$ is a tridiagonal non-singular SDD matrix and $a_{i,i}>0$ then $A^{-1}= 
[n_{i,k}]$ exists and 
\begin{itemize}
 \item $sign(n_{i,i})=1$
 \item $sign(n_{i,k})=(-1)^{i+k} \prod^{i}_{l=k+1}a_{l,l-1}, i>k$ 
 \item $sign(n_{i,k})=(-1)^{i+k} \prod^{k}_{l=i}a_{l,l+1}, i<k$
\end{itemize}
\end{theorem}

 Since the diagonal entries of $I-Q$ are strictly positive, according to 
Theorem~\ref{thm-signinv} the diagonal entries of $N$ are also positive. 
The non-diagonal entries of $I-Q$ are all negative thus the series 
multiplication from Theorem~\ref{thm-signinv} for $i>k$ reduces to 
$(-1)^{i+k+i-k}=(-1)^{2i}=1$. Similarly for the case  $i<k$ the multiplication 
reduces to $(-1)^{i+k+k-i}=(-1)^{2k}=1$.

\begin{proof}
Starting from Inequality~\ref{eq-invbound} and applying 
Theorem~\ref{thm-maininv} we obtain:
\begin{align*}
 E[T_{0}^{j}] &\leq  n_{1,1} + \bo{\mu^2} \leq \frac{1}{a_{1,1}+ 
a_{1,2}\xi_{2}} 
+ \bo{\mu^2}\\&= \frac{1}{p_{0,m}+p_{0,1}(1-\xi_{2})} +\bo{\mu^2}
\\\leq& \frac{1}{p_{0,m}+p_{0,1}(1-\xi_{2})}+ \bo{\mu^2}\\ &\leq 
\frac{1}{\frac{ (n-j)p_1}{n  }+\frac{\mu}{(\mu+1)} \frac{2j 
(n-j)p_2  }{n^2 }(1-\xi_{2})}+ \bo{\mu^2}\\
&=\frac{n}{n-j} \frac{1}{p_1+ \frac{\mu}{(\mu+1)} \frac{2 j  p_2  }{n  } 
(1-\xi_{2})} +o(\log{n}) 
\end{align*} 
where we substitute $p_{0,1}$ and $p_{0,m}$ with their values declared 
in Definition~\ref{def-mark}.
The definitions of $\xi_i$ and $\xi_m$ are obtained by simply substituting the 
matrix entries in Theorem~\ref{thm-maininv} with their respective values from 
Equations~\ref{eq-matfirst} to \ref{eq-matlast}.
\end{proof}

\subsection{Proofs of main results}

\begin{proof}[Proof of Theorem~\ref{thm-main}]
 Combining Lemma~\ref{lem-relate} and Lemma~\ref{lem-absbound} we obtain:
 \begin{align}
 E[T]&\leq o(n\log{n})+(1+o(1))\sum_{j=0}^{n-1}E[T^{j}] \nonumber\\ &\leq  
o(n\log{n})+ (1+o(1))\sum_{j=0}^{n-1}\left( \frac{n}{n-j} \frac{1}{p_1+ 
\frac{\mu}{(\mu+1)} \frac{2 j  p_2  }{n  } 
(1-\xi_{2})} \right) 
\label{eq-proofmain}
\end{align}
We will now divide the sum into two smaller sums:
\begin{align*}
 &\sum_{j=1}^{n}\frac{n}{n-j} \frac{1}{p_1+ \frac{\mu}{(\mu+1)} \frac{2 j  p_2  
}{n  } 
(1-\xi_{2})} 
\\ &= \sum_{j=1}^{n- n/\sqrt{\log{n}}} 
\frac{n}{n-j} \frac{1}{p_1+ \frac{\mu}{(\mu+1)} \frac{2 j  p_2  }{n  } 
(1-\xi_{2})} \\ &+ \sum_{j=n- 
n/\sqrt{\log{n}}+1}^{n} 
\frac{n}{n-j} \frac{1}{p_1+ \frac{\mu}{(\mu+1)} \frac{2 j  p_2  }{n  } 
(1-\xi_{2})} \\
&\leq \bo{n \sqrt{\log{n}}} + \frac{n}{p_1+ \frac{2   p_2  \mu}{(\mu+1)} 
\left(1-\frac{1}{\sqrt{\log{n}}}\right)
(1-\xi_{2})} \sum_{j=n- 
n/\sqrt{\log{n}}+1}^{n} 
\frac{1}{n-j} \\
&\leq \bo{n \sqrt{\log{n}}} +\frac{n \ln{n}}{p_1+ \frac{2   p_2  \mu}{(\mu+1)} 
\left(1-\frac{1}{\sqrt{\log{n}}}\right)
(1-\xi_{2})} 
 \end{align*}
We conclude the proof by substituting the sum in Eq~\ref{eq-proofmain} with the above expression.
\begin{align*}
E[T] &\leq o(n\log{n})+ (1+o(1))\left(\bo{n \sqrt{\log{n}}} + \frac{n 
\ln{n}}{p_1+ \frac{2   p_2  \mu}{(\mu+1)} 
\left(1-\frac{1}{\sqrt{\log{n}}}\right)
(1-\xi_{2})} \right)\\
&\leq o(n\log{n})+ (1+o(1))\left(\bo{n \sqrt{\log{n}}} +(1+o(1))\frac{n 
\ln{n}}{p_1+ \frac{2   p_2  \mu}{(\mu+1)} 
(1-\xi_{2})} \right)\\
&=(1+o(1))n \ln{n}\frac{1}{p_1+ \frac{2   p_2  \mu}{(\mu+1)} 
(1-\xi_{2})}
\end{align*} 

The expressions for $\xi_i$ and $\xi_m$ come from Lemma~\ref{lem-absbound} and 
proves the first statement.

For the second statement, we adapt the result for the variant of \muga which 
does not evaluate copies of either parents. When there is no diversity in the 
population the offspring is identical to the parent with probability $p_0$. 
Then, given that a fitness evaluation occurs, the probability of improvement via 
mutation is $p_{0,m}/(1-p_0)$ and the probability that diversity 
is introduced is $p_{0,1}/(1-p_0)$. The proof is identical to the 
proof of first statement, except for using probabilities 
$p^{*}_{0,m}=p_{0,m}/(1-p_0)$ and 
$p^{*}_{0,1}=p_{0,1}/(1-p_0)$  instead of $p_{0,1}$ and $p_{0,m}$ 
from Definition~\ref{def-mark}. Even if we pessimistically assume that a 
function evaluation occurs at every iteration when there is diversity in the 
population, we still get a $(1-p_0)$ decrease in the leading constant. 
\end{proof}

\end{document}